\documentclass{article}
\usepackage{spconf,amsmath,graphicx}
\usepackage{amsmath,amsfonts,amssymb,amsthm,amscd,array,latexsym}
\usepackage[caption=false,font=normalsize,labelfont=sf,textfont=sf]{subfig}
\usepackage{textcomp}
\usepackage{stfloats}
\usepackage{url}
\usepackage{verbatim}
\usepackage{nccmath}
\usepackage{graphicx}
\usepackage{cite}
\usepackage[hidelinks]{hyperref}      
\usepackage[resetlabels]{multibib}
\usepackage{url}            
\usepackage{booktabs}       
\usepackage{nicefrac}      
\usepackage{placeins}
\usepackage{multirow}
\usepackage{bm}
\usepackage{bbm}
\usepackage{siunitx}
\usepackage{enumitem}
\usepackage[ruled]{algorithm}
\usepackage[]{algorithmicx}
\usepackage[noend]{algpseudocode}
\usepackage[dvipsnames]{xcolor}
\usepackage[percent]{overpic}
\usepackage{standalone}
\usepackage{xcolor}
\usepackage{soul}
\usepackage{enumerate}
\usepackage{mathrsfs}
\usepackage{overpic}
\usepackage{mathtools}
\usepackage{graphicx,subfig,wrapfig}
\usepackage{times}
\usepackage{psfrag,epsfig}
\usepackage{comment}
\usepackage{tabularx}
\usepackage{setspace}
\usepackage{color}

\newtheorem{thrm}{Theorem}

\newtheorem*{prf*}{Proof}

\def\Cc{\mathcal{C}}

\DeclareMathOperator{\expect}{\mathbb{E}}

\DeclareMathOperator{\fdp}{\mathrm{FDP}}
\DeclareMathOperator{\fdr}{\mathrm{FDR}}
\DeclareMathOperator{\tpr}{\mathrm{TPR}}

\title{False Discovery Rate Control for Gaussian Graphical Models via Neighborhood Screening}

\name{Taulant Koka, Jasin Machkour, Michael Muma \thanks{T. Koka (taulant.koka@tu-darmstadt.de) and M. Muma (michael.muma @tu-darmstadt.de) have been funded by the ERC Starting Grant ScReeningData under grant number 101042407. J. Machkour (jasin.machkour@tu-darmstadt.de) has been supported by the LOEWE initiative (Hesse, Germany) within the emergenCITY center.}}
\address{Robust Data Science Group - Technische Universität Darmstadt\\64283 Darmstadt, Germany}

\begin{document}

\maketitle

\begin{abstract}
Gaussian graphical models emerge in a wide range of fields. They model the statistical relationships between variables as a graph, where an edge between two variables indicates conditional dependence. Unfortunately, well-established estimators, such as the graphical lasso or neighborhood selection, are known to be susceptible to a high prevalence of false edge detections. False detections may encourage inaccurate or even incorrect scientific interpretations, with major implications in applications, such as biomedicine or healthcare. In this paper, we introduce a nodewise variable selection approach to graph learning and provably control the false discovery rate of the selected edge set at a self-estimated level. A novel fusion method of the individual neighborhoods outputs an undirected graph estimate. The proposed method is parameter-free and does not require tuning by the user. Benchmarks against competing false discovery rate controlling methods in numerical experiments considering different graph topologies show a significant gain in performance.
\end{abstract}


\vspace{-10pt}
\section{Introduction}
\vspace{-10pt}
\label{sec:intro}
Graphical structures describe various complex phenomena in diverse fields, such as biology and healthcare \cite{Bessonneau2021,Iqbal2016}, climate science \cite{Zerenner2014}, or psychology and social sciences\cite{Norbury2020,Bhushan2019}. Moreover, by defining signals on the vertices of a graph $\mathcal{G}=(\mathcal{V},\mathcal{E})$ with vertex set $\mathcal{V}=\{1,\dots,p\}$ and edge set $\mathcal{E}\subseteq\mathcal{V}^2$ rather than on Euclidean space, the field of graph signal processing (GSP) provides powerful tools for the analysis of signals that can exceed the performance of classical approaches \cite{Ortega2018,Mateos2019}. However, the graph $\mathcal{G}$ is often not fully observable {\em a priori}, thus requiring its inference from noisy observations made at its vertices\cite{Maretic2017, Mei2016, Giannakis2018}. Consequently, errors arising from the graph's inference \cite{Miettinen2021} may lead to a serious issue in terms of interpretability and reproducibility. Such errors may encourage misguided or even incorrect scientific conclusions, which could severely impact, e.g., safety-critical applications like disease spread estimation \cite{Shirley2005} or neuroimaging \cite{Drakesmith2015}. 

\vspace{-1pt}
Many applications rely on undirected graphical models, where the nodes are represented by random variables $X_1,\dots,X_p$. For such a graph, $\mathcal{E}$ describes the conditional independence relationships between $X_1, \dots, X_p$. 
Learning $\mathcal{E}$ relies on the fact that every $(i,j)\in\mathcal{E}$ (with $i\neq j$) implies that 
$X_i$ and $X_j$ are dependent given $\{X_\ell:\ell \in \mathcal{V}\backslash\{i,j\}\}$. The joint distribution of $X_1,\dots,X_p$ is then said to be Markovian with respect to $\mathcal{G}$. In particular, the $p$-dimensional Gaussian distribution $\mathcal{N}_p(\bm \mu,\bm \Sigma)$ with mean $\bm\mu$ and covariance matrix $\bm \Sigma$ satisfies the Markov property, resulting in the well-known Gaussian graphical model (GGM). The GGM's peculiarity is that the graph $\mathcal{G}$ is encoded in its precision matrix $\bm \Omega = \bm \Sigma^{-1}$, i.e., $(i,j)\notin\mathcal{E}$ with $i \neq j$ if and only if $\bm \Omega_{i j}=0$. This property is exploited in many estimators that depend on partial correlation testing \cite{Drton2004} or enforce a sparse structure on $\bm\Omega$, e.g., in the graphical lasso (GLasso) \cite{Yuan_GLasso_2007},\cite{Friedman_GLasso_2007} or in neighborhood selection \cite{Meinshausen_Graphs_2006}. These methods are however unable to control the rate of falsely discovered edges in $|\widehat{\mathcal{E}}|$.

\vspace{-1pt}
We therefore aim to find an estimator $\widehat{\mathcal{E}}$ that controls the false discovery rate: $\fdr\coloneqq\expect[\fdp]$, provided that $n$ i.i.d. observations of $ (X_1,\dots,X_p)$ are available. The $\fdr$ is the expected value of the false discovery proportion: $\fdp \coloneqq V/(R\vee1)$, where $V\coloneqq|\widehat{\mathcal{E}}\backslash\mathcal{E}|$ is the number of falsely discovered edges, $R=|\widehat{\mathcal{E}}|$ (with $|\cdot|$ being the cardinality of a set) is the number of selected edges, and $R\vee1 = \max\{R,1\}$ \cite{BH_FDR_1995}. We say that the $\fdr$ is controlled at level $\alpha$ if $\fdr \leq \alpha$ for some $\alpha\in[0,1]$. At the same time, our goal is to achieve a high true positive rate ($\tpr$), i.e., statistical power. It is defined as $\tpr\coloneqq\expect[S/|\mathcal{E}|]$, where $S=|\widehat{\mathcal{E}}\cap\mathcal{E}|$ is the number of correctly selected edges. 
Research addressing $\fdr$ control in graphical models remains limited \cite{Weidong2013, Maathuis_GKF_2021, Yu_FDR_bio_2021, Li2008}, with several existing methods facing challenges related to computational complexity and reduced power for small sample sizes. 

\vspace{-1pt}
Therefore, we introduce a novel approach for $\fdr$ control in Gaussian Graphical Models (GGMs). Our method builds upon the recently introduced Terminating-Random Experiments (T-Rex) framework for fast high-dimensional $\fdr$ control in variable selection\cite{machkour_trex_2022,machkour_screen_2023}, which offers provable $\fdr$ control for the estimated edge set in GGMs. An implementation of the proposed method is available in the open-source R package `TRexSelector'\cite{machkour2022TRexSelector}.
\vspace{-10pt}
\section{Methodology}
 \label{sec:meth}
 \vspace{-10pt}
 Estimation of GGMs is frequently done by partial correlation testing\cite{Drton2004}, penalized maximum likelihood estimation (GLasso)\cite{Yuan_GLasso_2007,Friedman_GLasso_2007}, or nodewise variable selection (neighborhood selection)\cite{Meinshausen_Graphs_2006}.
We are particularly interested in the latter approach, which can be interpeted as a pseudo-likelihood method for precision matrix estimation. Its key advantages are its low computational complexity compared to the GLasso and the possibility of massively parallelizing the procedure. It can also be applied in high-dimensional settings for sparse graphs, unlike partial correlation testing since the sample covariance matrix is singular for $p>n$\cite{Meinshausen_Graphs_2006,Drton2004}.

Let $\bm X = (\boldsymbol{x}_1 \dots \boldsymbol{x}_p)\in \mathbb{R}^{n\times p}$ be the data matrix with $n$ i.i.d. observations from a $p$-dimensional zero-mean Gaussian distribution $\mathcal{N}_p(\bm 0,\bm\Sigma)$, and let $\mathcal{G}=(\mathcal{E},\mathcal{V})$ be the associated undirected graph. In neighborhood selection, one considers $p$ independent nodewise regression problems
\begin{align}
\label{eq:nodewisereg}
    \bm x_i = \bm X_{\mathcal{I}}\bm \beta_i+\bm\epsilon_i, \:\: i=1,\dots,p,
\end{align}
where $\bm X_{\mathcal{I}}$ is the column submatrix that results from keeping only those columns in $\bm X$ whose indices are in the index set $\mathcal{I}=\mathcal{V}\backslash i$. It is well-known that $(i,j)\notin\mathcal{E}$ if and only if $[\bm\beta_i]_j = [\bm\beta_j]_i = 0$, where $[\cdot]_i$ and $[\cdot]_{ij}$ denote the $i$th and $ij$th entry of a vector and a matrix, respectively. Assuming that $|\mathcal{E}|\ll p^2$, i.e., $\mathcal{G}$ has a sparse structure, the problem can thus be related to that of $p$ independent variable selection problems under sparsity constraints. The estimate $\widehat{\mathscr{N}}_i$ of the true neighborhood $\mathscr{N}_i$ associated with node $i\in\mathcal{V}$ can then be deduced by applying the OR-rule or the AND-rule \cite{Meinshausen_Graphs_2006}.

\vspace{-7.5pt}
\subsection{The Terminating-Random Experiments Framework}
\label{ssec:screen_t_rex}
\vspace{-5pt}
Recently, the Terminating-Random Experiments (T-Rex) selector \cite{machkour_trex_2022}, a framework for fast high-dimensional variable selection with provable $\fdr$ control has been proposed. The $\fdr$ is controlled (see Theorem 1 of \cite{machkour_trex_2022}) by conducting multiple early terminated random experiments, in which computer-generated dummy variables compete with real variables in a forward variable selection, and  subsequently fusing the resulting candidate sets. The dummies are drawn from any univariate probability distribution with finite mean and variance, e.g., a standard normal distribution (see Theorem 2 of \cite{machkour_trex_2022}). The computational complexity of the T-Rex is $O(np)$ and it scales to high-dimensional variable selection problems with millions of variables.

\textbf{Screen-T-Rex:}
In \cite{machkour_screen_2023}, the Screen-T-Rex selector, a computationally cheap variant of the T-Rex selector, has been proposed for fast screening of large-scale data banks. In contrast to the standard T-Rex selector, which calibrates its parameters such that the number of selected variables is maximized while controlling the $\fdr$ at an {\em a priori} defined target level, the Screen-T-Rex selector stops the forward selection in each random experiment as soon as one dummy variable has been selected and outputs a self-estimated conservative $\fdr$ estimate (see Theorem 1 of \cite{machkour_screen_2023}).

\vspace{-7.5pt}
\subsection{FDR Controlled GGMs: Screening Neighborhoods}
\vspace{-5pt}
Consider the nodewise regression model in \eqref{eq:nodewisereg}. We propose to solve the $p$ variable selection problems with the Screen-T-Rex selector, which conducts $K$ random experiments, resulting in candidate sets $\{\mathcal{C}_{k,i}\}_{k =1}^K$ for every neighborhood $\mathscr{N}_i$. We denote the relative occurrences of each edge by $\Phi(i,j) \coloneqq  \sum_{k = 1}^{K} \mathbbm{1}_{k}(i,j)/K,$ where $\mathbbm{1}_{k}(i,j) = 1 $ if $ j \in \Cc_{k,i}$, and $\mathbbm{1}_{k}(i,j) = 0 $ otherwise. The estimated neighborhood of node $i \in \mathcal{V}$ is obtained by thresholding the relative occurences, i.e., $\widehat{\mathscr{N}}_i=\{(i,j):\Phi(i,j)>0.5\}$ (see \cite{machkour_screen_2023}). For each $\widehat{\mathscr{N}}_i$, let us now denote the number of falsely selected nodes by $V_{i} \coloneqq \big| \widehat{\mathscr{N}}_i\:\:\backslash\mathscr{N}_i \big|$, the number of correctly selected nodes by $S_{i} \coloneqq \big|\widehat{\mathscr{N}}_i\:\:\cap\mathscr{N}_i \big|$, and the total number of selected nodes by  $R_{i} \coloneqq V_{i} + S_{i}.$ For every node $i$, we have a total of $p-1 = p_{0,i}+p_{1,i}$ variables to select from, where $p_{1,i}$ and $p_{0,i}$ are the total number of true active and null variables, respectively.

The number of selected null variables $V_i$ in every neighborhood estimate $\widehat{\mathscr{N}}_i$ is a random variable that is stochastically dominated by a random variable following the negative hypergeometric distribution $\mathrm{NHG}(p_{0,i} + p -1, p_{0,i}, 1)$ with $(p_{0,i}+ p - 1)$ total elements, $p_{0,i}$ success elements, and one failure after which a random experiment is terminated. It describes the process of randomly picking variables without replacement, with equal probability, and one at a time, from the combined set of $p_{0,i}$ null and $p-1$ dummy variables. This process adequately describes the forward variable selection process considered in the T-Rex framework \cite{machkour_trex_2022}. Thus, the expectation of $V_i$ can be upper-bounded as $\expect[V_i]\leq p_{0,i}/(p-1 + 1)$, where the right hand side of the inequality is the expected value of $\mathrm{NHG}(p_{0,i} + p -1, p_{0,i}, 1)$ (see \cite{machkour_trex_2022} for more details).
\begin{algorithm}[b]
    \caption{Neighborhood Screening}
    \label{algo:unsymmetric}
    \hspace*{\algorithmicindent}\textbf{Input:}
    $\bm X = [\boldsymbol{x}_1,\dots,\boldsymbol{x}_p] \in \mathbb{R}^{n\times p},\:\widehat{\mathcal{E}}=\varnothing$ 
    \begin{algorithmic}[1]
        \State\textbf{for $i=1,\dots,p$} 
        \State\hspace{\algorithmicindent}$\{\Phi(i,j)\}_{j=1}^p \gets$ Run Screen-T-Rex\cite{machkour_screen_2023} on the\Statex\hspace{\algorithmicindent}  regression problem: $\:\bm x_i = \bm X_{\mathcal{I}}\bm \beta_i+\bm\epsilon_i$ (see \eqref{eq:nodewisereg})
        \State\hspace{\algorithmicindent} $\:\widehat{\mathcal{E}}\gets\widehat{\mathcal{E}}\cup\{(i,j):\Phi(i,j) > 0.5\}$
        \State$\hat{\alpha} \gets p/(|\widehat{\mathcal{E}}|\vee 1)$ (Determine $\fdr$ estimate)
    \end{algorithmic}
    \hspace*{\algorithmicindent} \textbf{Output:} Selected edge set $\widehat{\mathcal{E}}$, estimated $\fdr$ $\hat{\alpha}$
\end{algorithm}

Note that in the true and estimated edge sets, i.e., $\mathcal{E}$ and $\widehat{\mathcal{E}}$, we count $(i,j)\in \mathcal{V}^2$ and $(j,i)\in \mathcal{V}^2$ as individual edges.
Our focus is on $\fdr$ control for the (potentially undirected) estimate $\widehat{\mathcal{E}}=\{(i,j): j\in\widehat{\mathscr{N}}_i\:\,\}_{i=1}^p$ that results from directly combining the individual neighborhoods. Hence, there is a one-to-one map between the $\widehat{\mathscr{N}}_1,\dots,\widehat{\mathscr{N}}_p$ and $\widehat{\mathcal{E}}$ and we have $V=|\widehat{\mathcal{E}}\backslash\mathcal{E}|=\sum_{i=1}^pV_i$ and $R = |\widehat{\mathcal{E}}|= \sum_{i=1}^pR_i$.
Our $\fdr$ estimator is then given by $\hat{\alpha} := \sum_{i=1}^p 1 /(R \vee 1)= p /(R \vee 1)$. We summarize the procedure in Algorithm~1.
\begin{thrm}   
\label{thrm:1}
Let $\hat{\alpha}:=p/(R\vee1)$. Algorithm~\ref{algo:unsymmetric} controls the $\fdr$ at the estimated level $\hat{\alpha}$, i.e., $\fdr = \expect[\fdp] \leq \hat{\alpha}$.
\end{thrm}
\begin{proof}
By definition of the $\fdr$, we have
    \begin{align*}
        \fdr &\coloneqq \expect\left[\fdp\right]=\expect\left[\frac{V}{R\vee1}\right]=\expect\left[\frac{\sum_{i=1}^p V_{i}}{R\vee1}\right]\\&=\frac{\hat{\alpha}}{p}\sum_{i=1}^p \expect\left[V_{i}\right]\leq\frac{\hat{\alpha}}{p}\,\frac{\sum_{i=1}^p p_{0,i}}{\left[(p-1)+1\right]}\leq\hat{\alpha}\frac{\sum_{i=1}^p p}{p^2}=\hat{\alpha},
    \end{align*}
where the second inequality follows from $V_{i}$ being stochastically dominated by the negative hypergeometric distribution $\mathrm{NHG}(p_{0,i}+p-1, p_{0,i},1)$ with $\expect[V_{i}] \leq p_{0,i}/p$.\footnote{Note that $R$ becomes observable for any fixed voting level, rendering it deterministic, thus $\hat{\alpha}(v = 0.5)$ is also deterministic, which allows us to move it out of the expectation.}
\end{proof}
Intuitively, it is clear that $\hat{\alpha}$ is conservative, since only one out of $p-1$ dummies is allowed to enter the solution paths of the $K$ random experiments for every $i$ before terminating the forward selection processes. We thus expect on average no more than one out of at most $p-1$ null variables to be included in each candidate set $\mathcal{C}_{k,i}, k =1,\dots,K$. Therefore, at most $p$ null variables are on average expected to be included among all selected variables.
\vspace{-10pt}
\subsection{Obtaining an Undirected Graph}
\vspace{-5pt}
It is clear that conditional independence graphs like the GGM are undirected. Therefore, the selected neighborhoods have to be fused in a way that produces an undirected graph. Unlike previous approaches that rely on the AND and OR rule \cite{Meinshausen_Graphs_2006}, we take a different approach. As previously described, every edge $(i,j)$ in the estimated neighborhood $\widehat{\mathscr{N}}_i$ of a node is assigned its relative occurence $\Phi(i, j)$, and is only selected if $\Phi(i,j) > 0.5$. Instead of counting $(i,j)$ and $(j,i)$ individually, they are assigned a joint relative occurence $\Phi_{{\mathrm{joint}}}(j,i)=\Phi_{\mathrm{joint}}(i,j)\coloneqq(\Phi(i,j)+\Phi(j,i))/2$. Thus, the undirected estimate of the edge set is given by $\widehat{\mathcal{E}}_\Phi = \{(i,j):\Phi_{\mathrm{joint}}(i,j)>0.5\}$. The method is summarized in Algorithm~\ref{algo:symmetric}.
\begin{algorithm}[h]
    \caption{Undirected Neighborhood Screening}
    \label{algo:symmetric}
    \hspace*{\algorithmicindent}\textbf{Input:}
    $\bm X = [\boldsymbol{x}_1,\dots,\boldsymbol{x}_p] \in \mathbb{R}^{n\times p}$ 
    \begin{algorithmic}[1]
        \State\textbf{for $i=1,\dots,p$} 
        \State\hspace{\algorithmicindent}$\{\Phi(i,j)\}_{j=1}^p \gets$ Run Screen-T-Rex\cite{machkour_screen_2023} on the\Statex\hspace{\algorithmicindent}  regression problem: $\:\bm x_i = \bm X_{\mathcal{I}}\bm \beta_i+\bm\epsilon_i$ (see \eqref{eq:nodewisereg})
        \State $\{\Phi_{\mathrm{joint}}(i,j)\}_{i,j=1}^p\gets\{(\Phi(i,j)+\Phi(j,i))/2\}_{i,j=1}^p$
        \State $\widehat{\mathcal{E}}_\Phi\gets\{(i,j)\in\mathcal{V}^2:\Phi_{\mathrm{joint}}(i,j) > 0.5\}$
        \State$\hat{\alpha} \gets p/(|\widehat{\mathcal{E}_\Phi}|\vee 1)$ (Determine $\fdr$ estimate)
    \end{algorithmic}
    \hspace*{\algorithmicindent} \textbf{Output:} Selected edge set $\widehat{\mathcal{E}}_\Phi$, estimated $\fdr$ $\hat{\alpha}$
\end{algorithm}
\vspace{10pt}

\vspace{-15pt}
\section{Numerical Simulations}
\label{sec:simul}
 \vspace{-10pt}
\begin{figure}[h]
    \centering
    \includegraphics[width=\linewidth]{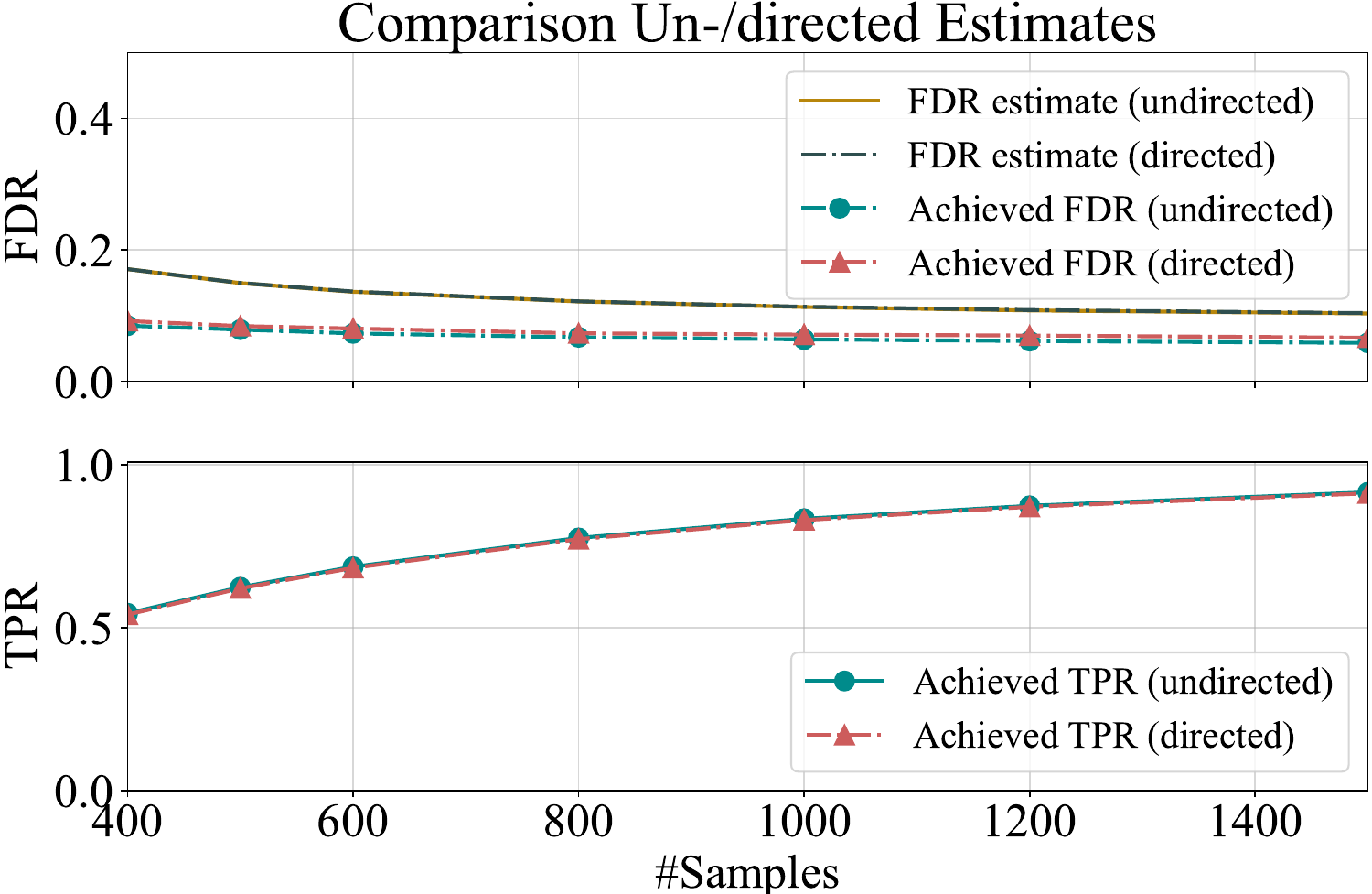}
    \caption{\small Comparison of the performance of Algorithm~\ref{algo:unsymmetric} and \ref{algo:symmetric} on an ER graph with an edge probability of $10\%$ and partial correlations $|\rho_{ij}|\in[0.2,0.6]$. The sample size varies between $400$ and $1500$. Both methods show a similar performance, except for a slightly smaller achieved $\fdr$ of the undirected graph estimator.}
    \label{fig:sym_unsym}
\end{figure}
In this section, we evaluate our proposed method in numerical simulations. First, we compare the performance of Algorithms~\ref{algo:unsymmetric} and~\ref{algo:symmetric} in terms of $\tpr$  and $\fdr$. Then, Algorithm~2 is benchmarked against other $\fdr$ controlling methods for structure estimation of GGMs. 
In all experiments, we draw $n$ independent samples from a multivariate Gaussian distribution $\mathcal{N}(\bm 0, \bm\Sigma)$, where the structure of the precision matrix $\bm\Omega$ follows an undirected graph $\mathcal{G}$ with adjacency matrix $\bm A$, i.e., $\bm A_{i j}=1$ if $(i,j)\in\mathcal{E}$ and $\bm A_{i j}=0$ otherwise. To ensure positive definiteness of $\bm\Omega$, we follow the approach of \cite{Maathuis_GKF_2021} and let $\bm\Omega \coloneqq \bm \Omega_0 + ( |\lambda_{\min}(\bm \Omega_0)| + 0.5)\bm I$, where $\lambda_{\min}(\cdot)$ denotes the minimum eigenvalue of a matrix. Given $\bm A$, we construct  $\bm\Omega_0$ in every setting as follows: $[\bm \Omega_0]_{i i}=1$ and $[\bm \Omega_0]_{i j}=\rho_{i j}\bm A_{i j}$ with $i\neq j,$ where all $\rho_{i j}$ are independently drawn from the uniform distribution on $[-0.6,-0.2]\cup[0.2,0.6]$ and $\rho_{i j}=\rho_{j i}$. In all experiments, we constrain ourselves to $p=100$ variables and $100$ Monte Carlo runs per setting.
\vspace{-7.5pt}
\subsection{Comparison of Algorithms~\ref{algo:unsymmetric} and \ref{algo:symmetric}}
\label{ssec:comp_alg12}
\vspace{-5pt}
We compare the two proposed methods in term of $\fdr$ and $\tpr$
for an Erd\H{o}s-Rényi (ER) model with adjacency matrix $\bm A$, where all elements $\bm A_{i j} = \bm A_{j i}$ with $ i\neq j$ are independent Bernoulli random variables with $\mathbb{P}[\bm A_{i j}=1]=0.1$. We evaluate the methods by varying the sample size $n$ between $400$ and $1500$. The results are presented in Fig.~\ref{fig:sym_unsym}, where we observe that the $\fdr$ estimates as well as the $\tpr$ estimates of both methods are virtually identical, and that the undirected graph estimator achieves a slightly smaller $\fdr$ than Algorithm~\ref{algo:unsymmetric}. As suggested by Theorem~\ref{thrm:1}, both methods also empirically control the $\fdr$ at the self-estimated level.
\vspace{-7.5pt}
\subsection{Benchmark on Varying Topologies}
\vspace{-5pt}
\begin{figure*}
    \includegraphics[width=0.33\linewidth]{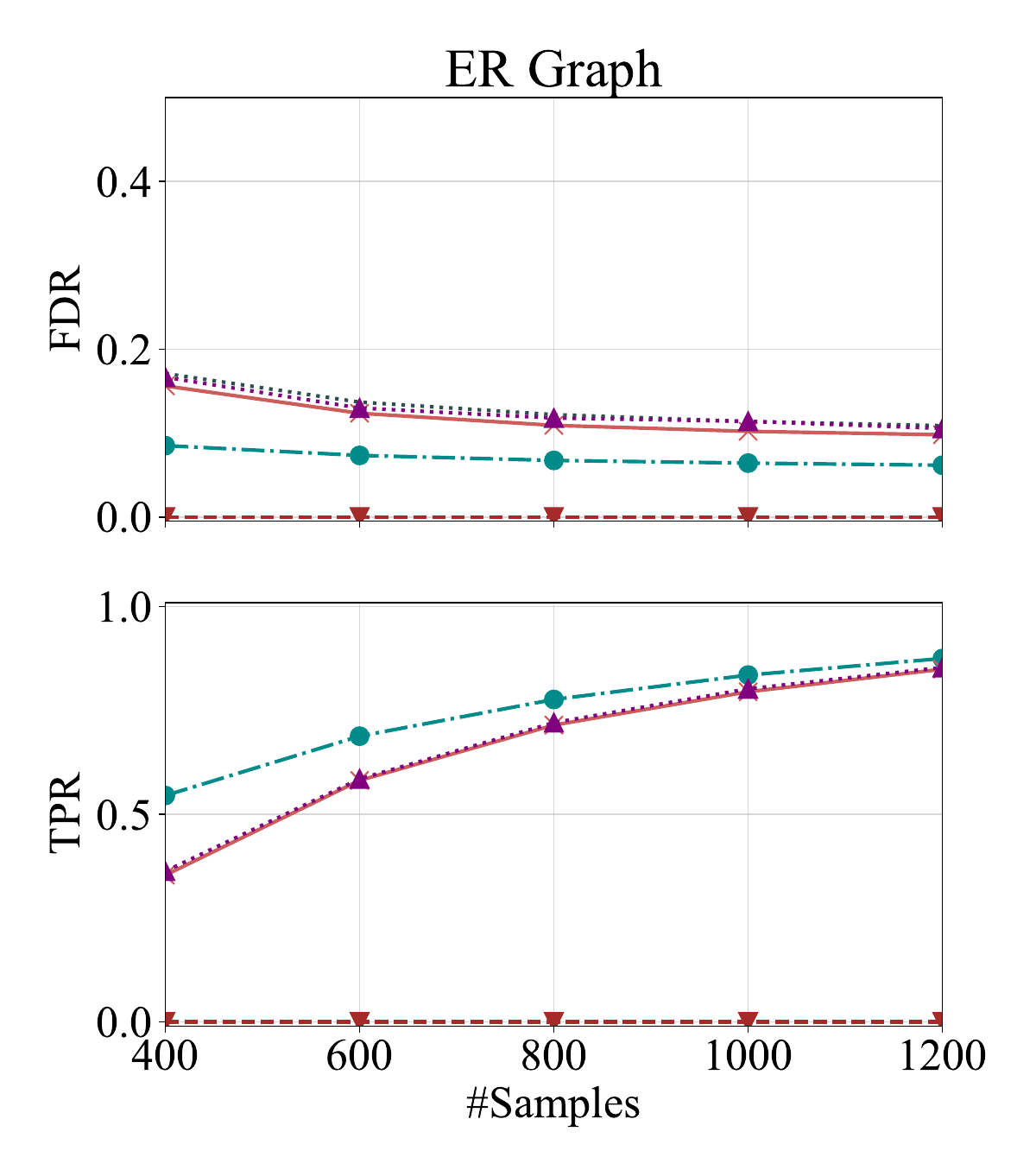}
    \includegraphics[width=0.33\linewidth]{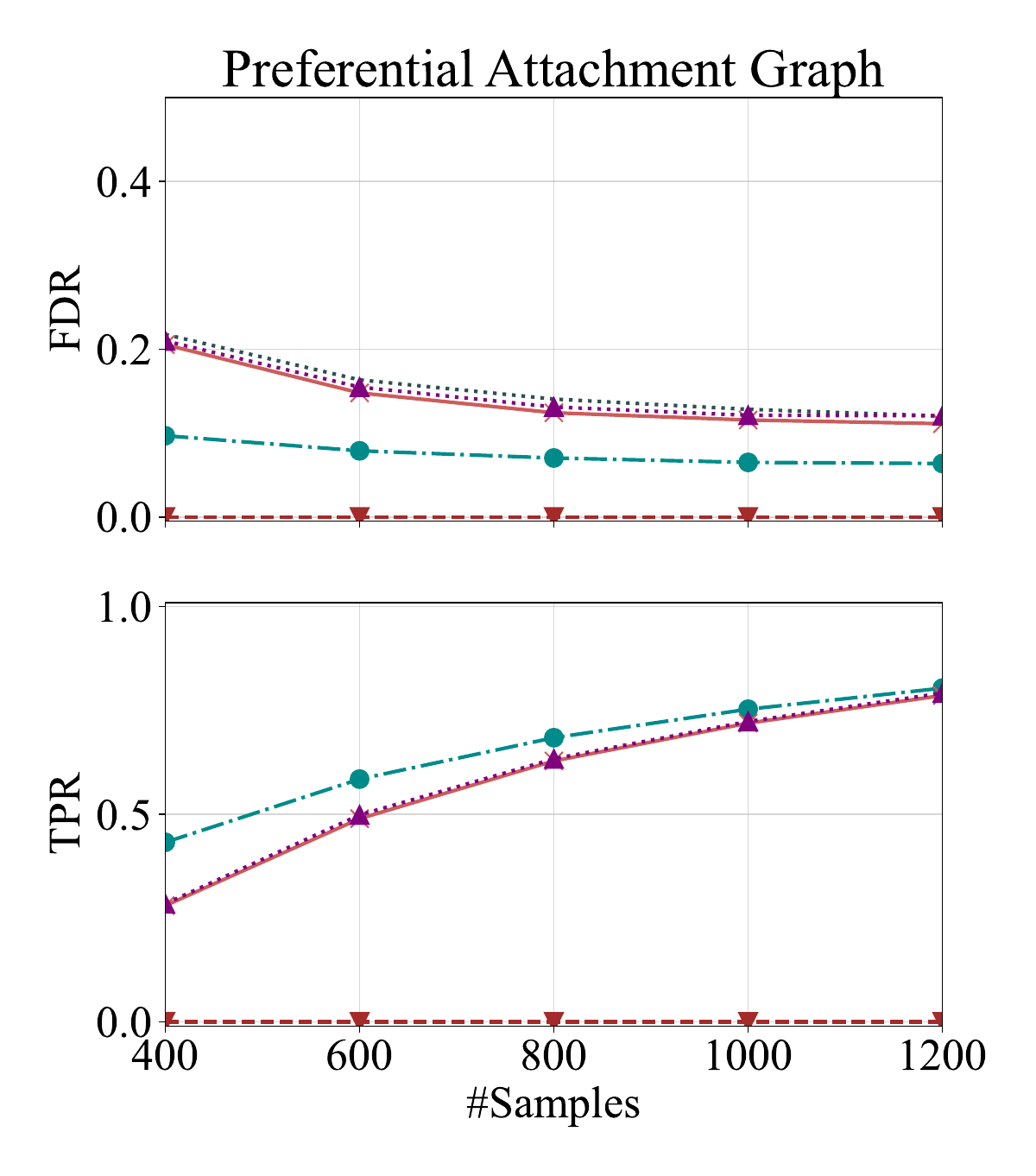}
    \includegraphics[width=0.33\linewidth]{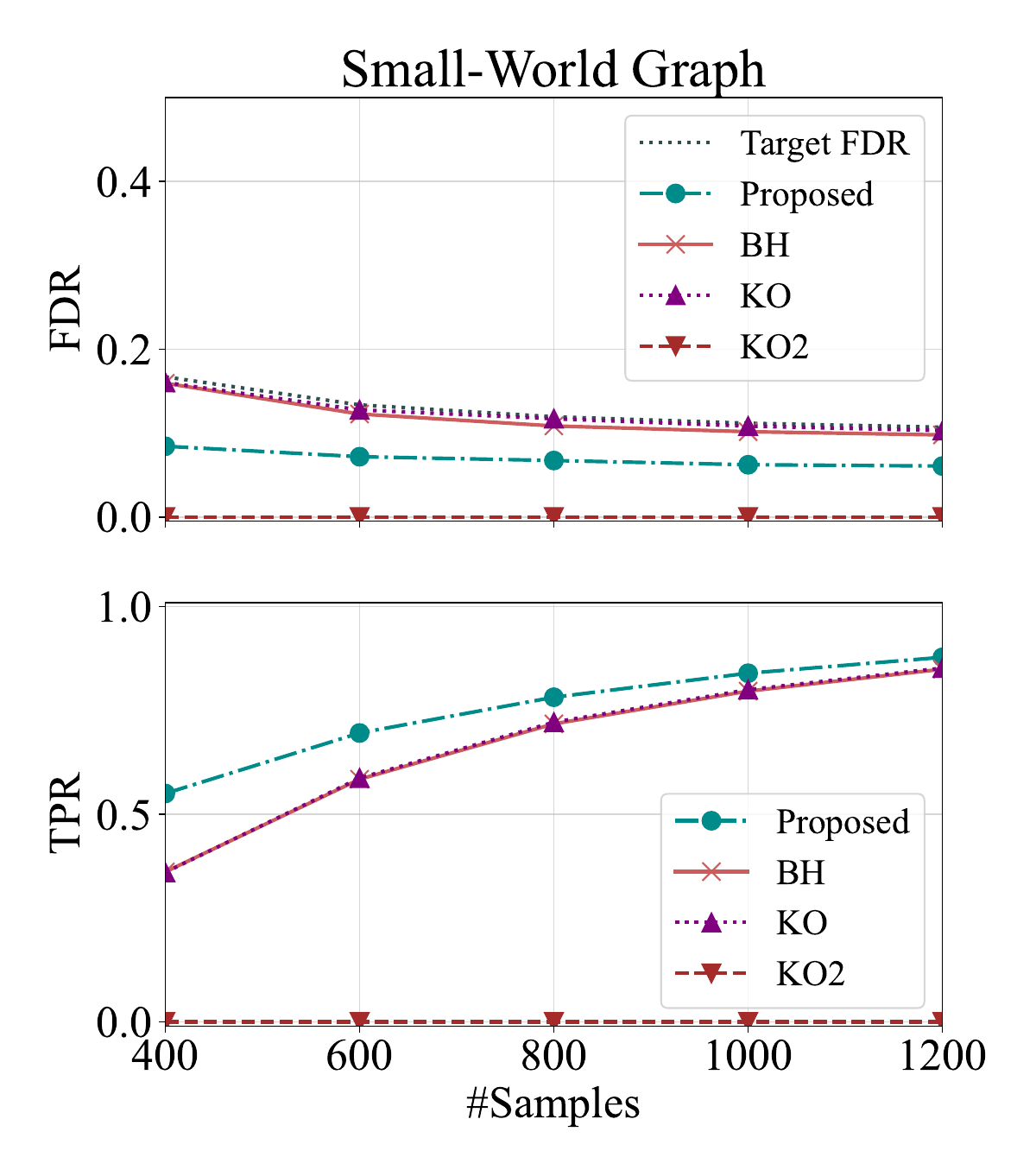}
    \caption{\small Comparison of the performance of Algorithm~\ref{algo:symmetric}, BH, KO and KO2 on an ER graph with an edge probability of $10\%$(left), a sub-linear preferential attachment graph with growth constant $m=5$ and a power law exponent of $0.5$ (middle), and a small-world graph with $2D=10$ neighbors per node and a rewiring probability of $0.5$ (right). Compared to the competing methods, the proposed method shows a significant gain in performance in all experiments.}
    \label{fig:simulation_graphs}
\end{figure*}
In the following, we benchmark our proposed method against competing methods that control the $\fdr$ in GGM structure estimation. In particular, the Benjamini-Hochberg (BH) method \cite{BH_FDR_1995} and two methods based on knockoffs proposed in \cite{Yu_FDR_bio_2021} (KO) and \cite{Maathuis_GKF_2021} (KO2). The BH method is applied to the p-values that are obtained from two-sided tests for zero partial correlations, while the KO method constructs knockoffs of edges that mimic the partial correlation structure of the data and KO2 applies the knockoff methodology to the nodewise regression approach of \cite{Meinshausen_Graphs_2006}. The precision matrix is obtained from a symmetric adjacency matrix, as previously described. 

We consider three different graph topologies: i) an ER model with edge probability of $10\%$; ii) a graph based on the Barabási-Albert model \cite{Barabasi1999}, which exhibits continuous growth by adding new vertices that preferentially attach to well-connected sites. Here, we consider sublinear preferential attachment, specifically a power law distribution with an exponent less than $1$. Sublinear growth leads to networks with stretched-exponential degree distribution, which are good models for, e.g., protein interaction networks \cite{Stumpf2005} or the neuronal network of {\em Caenorhabditis (C.) elegans} \cite{Varshney2011}. In our case, we set the power to $1/2$, and each growth step adds $m=5$ edges to the graph. iii) Lastly, we examine small-world graphs using the Watts-Strogatz model \cite{Watts1998}. Initially, a ring lattice with $p$ nodes is created, where each node connects to its $D$ rightmost neighbors, ensuring that $0 < (j-i)\mathrm{mod}(p)\leq D$ (where $\mathrm{mod}$ is the modulo operator), and resulting in $2D$ neighbors per node. Subsequently, edges are rewired with some probability, replacing $(i,j)$ with $(i,\ell)$, where $\ell$ is chosen uniformly at random from all nodes except $i$, while avoiding duplicate edges. The neuronal network of {\em C. elegans} exhibits small-world properties as well \cite{Varshney2011,Newman_2000}, but small-world graphs also emerge in several other real-world networks, such as collaboration networks between actors or the western power grid of the United States \cite{Newman_2000}. Here, we set the rewiring probability to $50\%$ and the number of neighbors per node to $2D = 10$. 

The results for varying sample size ($n$  from $400$ to $1200$), are shown in Fig.~\ref{fig:simulation_graphs}. In order to ensure a fair comparison, the target level of the competing methods is set to be the $\hat{\alpha}$ resulting from Algorithm~\ref{algo:symmetric}. For all considered topologies, we observe that each method empirically controls the $\fdr$ at the target level. The proposed method shows superior performance in terms of the $\tpr$ in all experiments, especially for smaller sample sizes, while being more conservative in its achieved $\fdr$. The results of the KO method and the BH method deviate only slightly in the achieved $\fdr$, and are virtually identical for the $\tpr$. In contrast, the KO2 method does not select any edges in all of the experiments.
\vspace{-10pt}
\vfill
\vspace{-5pt}
\section{Conclusion}
\label{sec:conc}
\vspace{-5pt}
We have presented a novel FDR controlling method for graphical models based on a nodewise variable selection approach. A novel fusion process of the individual neighborhoods yields an undirected graph estimate, which has proven to be useful for $\fdr$ controlled GGM estimation in numerical experiments. Benchmarks against competing methods show a significant gain in performance, making it a promising new graph structure estimator for Gaussian graphical models.

\vfill\clearpage\small
\bibliographystyle{IEEEbib}
\bibliography{bibliography.bib}

\end{document}